\documentclass[twocolumn,times,final]{elsarticle}

\usepackage{framed,multirow}

\usepackage{amssymb}
\usepackage{amsmath}
\usepackage{amsthm}
\usepackage{latexsym}

\usepackage{url}
\usepackage{xcolor}
\usepackage[numbers]{natbib}
\definecolor{newcolor}{rgb}{.8,.349,.1}

\newtheorem{theorem}{Theorem}

\begin{document}
\begin{frontmatter}

\title{An Exact Finite-dimensional Explicit Feature Map for Kernel Functions}

\author[1]{Kamaledin Ghiasi-Shirazi\corref{cor1}} 
\cortext[cor1]{Corresponding author: 
  Tel.: +98-051-3880-5158}
\ead{k.ghiasi@um.ac.ir}

\author[2]{Mohammadreza Qaraei}
\ead{mohammadreza.mohammadniaqaraei@aalto.fi}

\affiliation[1]{organization={Department of Computer Engineering, Faculty of Engineering ,Ferdowsi University of Mashhad (FUM)},
                city={Mashhad}, 
                country={Iran}}

\affiliation[2]{organization={Department of Computer Science, Aalto University},
                city={Helsinki}, 
                country={Finland}}

\begin{abstract}
Kernel methods in machine learning use a kernel function that takes two data points as input and returns their inner product after mapping them to a Hilbert space, implicitly and without actually computing the mapping. For many kernel functions, such as Gaussian and Laplacian kernels, the feature space is known to be infinite-dimensional, making operations in this space possible only implicitly. This implicit nature necessitates algorithms to be expressed using dual representations and the kernel trick.
In this paper, given an arbitrary kernel function, we introduce an explicit, finite-dimensional feature map for any arbitrary kernel function that ensures the inner product of data points in the feature space equals the kernel function value, during both training and testing. The existence of this explicit mapping allows for kernelized algorithms to be formulated in their primal form, without the need for the kernel trick or the dual representation. 
As a first application, we demonstrate how to derive kernelized machine learning algorithms directly, without resorting to the dual representation, and apply this method specifically to PCA. As another application, without any changes to the t-SNE algorithm and its implementation, we use it for visualizing the feature space of kernel functions.
\end{abstract}


\end{frontmatter}


\section{Introduction}
\label{sec:introduction}
Kernel methods are one of the important techniques in machine learning, which alongside and in combination with other methods such as Bayesian learning, neural networks, and deep learning play a crucial role in forming suitable solutions for machine learning problems. The essence of kernel methods is the kernel trick that involves replacing all inner products in a machine learning algorithm with a kernel function. 
Kernel functions have the property that, for each of them, there exists a Hilbert space and a mapping from the input space to that Hilbert space in such a way that the value of the kernel function for any pair of input values equals the inner product of the mapped inputs in that Hilbert space. 
Specifically, for each kernel function \( k: X \times X \rightarrow \mathbb{R} \) defined over an input space \( X \), there exist a Hilbert space \( H \) and a mapping \( \phi: X \rightarrow H \)  such that for every \( x, z \in X \), we have \( k(x,z) = \langle \phi(x), \phi(z) \rangle_H \).

An important point is that kernel functions are usually simple mathematical expressions whose values equal the inner product of their inputs after mapping to a Hilbert space, without performing the mapping explicitly. For example, the Gaussian kernel function, expressed with the simple formula 
\begin{eqnarray} 
k(x,z) = \exp\left(-\frac{1}{2\sigma^2} \|x-z\|^2\right),
\end{eqnarray}
corresponds to an infinite-dimensional Hilbert space.

For various reasons, it has always been taken form granted that, for a general kernel function, an explicit finite-dimensional mapping of data to the corresponding Hilbert space does not exist. Firstly, many well-known Hilbert spaces associated with kernel functions, such as RKHS (Reproducing Kernel Hilbert Spaces) \citep{Aronszajn1950} and the Hilbert space of the Mercer's theorem \citep[page 36][]{Scholkopf2002}, are spaces of functions or infinite sequences which are not computationally suitable for explicit mapping. Secondly, except in special cases, Hilbert spaces corresponding to kernel functions are typically infinite-dimensional, leading to the belief that explicit computation of such mappings is not feasible from a computational perspective.

If the computation of an explicit feature mapping were possible, we could gain a sense of the feature spaces of kernel functions, helping us in a better selection of kernel functions for machine learning algorithms\footnote{Many methods have been proposed for obtaining an explicit feature mapping, independent of the number of data points, to reduce the computational burden of kernel methods \citep{Francis2021}. However, for general kernel functions, these feature mappings are only approximations.}. The existence of an exact explicit feature space facilitates expressing the learning algorithms in their primal form, eliminating the need for the dual form. Consequently, the parameters of a learning machine can be learned directly without expressing the solution in terms of training data.

One notable method for approximating the feature mapping of kernel functions is the Nyström method \citep{Williams2001}. This method introduces an explicit feature mapping by restricting the kernel expansion to a finite number of terms with the largest eigenvalues and introducing a data-dependent approximation of the eigenfunctions of the kernel function.

Another prominent approach in approximating the feature mapping of kernel functions is the random Fourier feature map method \citep{Rahimi2007}. This method approximates the feature mapping of translation-invariant kernel functions efficiently using sampling from the Fourier transform of these functions.
These approximation methods were extended to families of additive kernels \citep{Vedaldi2012} and to one-shot similarity kernels \citep{Zafeiriou2013}. 
Additionally, to reduce feature-space dimensions, \citep{Hamid2014} suggested to randomly project vectors into a low-dimensional feature space. 
Despite the computational improvements of kernel-based algorithms in the mentioned methods, all these approaches compute the feature mapping of kernel functions approximately, leading to the mindset that 
an exact computation of the feature mapping is infeasible.

In this article\footnote{This work has been previously published in Persian \citep{Ghiasi2022}}, we propose a method for computing an exact explicit feature mapping for any desired kernel function in kernel-based machine learning methods. More precisely, for each machine learning algorithm trained on training data \( x^{(1)}, \ldots, x^{(N)} \), we introduce a finite-dimensional Hilbert space and an explicit mapping such that the value of the kernel function between an arbitrary input and a training data point equals the inner product between them in that feature space. It is important to note that in any machine learning algorithm, one of the inputs to the kernel function is always a training data point. Specifically, in the training phase, both inputs to the kernel function are from the training data, while during testing, one input is from the training data, and the other is from the testing data. Therefore, the proposed method, which assumes one input is always from the training data, suffices for calculating an explicit and exact feature mapping of kernel functions in machine learning algorithms.

We should mention that \citep{Francis2020} also claims the proposal of an explicit and exact feature mapping. They first introduce an explicit feature mapping and then, based on the inner product in that space, define a specific kernel function. By construction, the introduced feature mapping is explicit for that kernel function. However, no explicit feature mapping was introduced for general kernel functions.

The rest of the paper continues as follows. In Section 2, we introduce the proposed explicit, exact, and finite-dimensional feature mapping. In Section 3, we derive the kernelized version of the PCA algorithm in both primal and dual forms using the introduced explicit feature mapping. We subsequently combine the primal and dual solutions to reach a new computationally more efficient solution. In Section 4, we visualize the feature space of kernel functions using t-SNE on the MNIST dataset, illustrating that the suitability or unsuitability of kernel functions for a dataset can be grasped by this visualization method.

\section{The proposed explicit, exact, and finite-dimensional feature mapping}
Let \( k \) be an arbitrary kernel function and \( \{x_1, \ldots, x_N\} \) be a given set of training data. 
Assuming that \( H \) is a Hilbert space corresponding to this kernel function, and \( \psi: X \rightarrow H \) is a mapping from the input space $X$ to this Hilbert space, then for any two arbitrary data points \( x, z \in X \), we have:
\begin{eqnarray} k(x, z) = \langle \psi(x), \psi(z) \rangle_H. \end{eqnarray}

We aim to find a finite-dimensional feature space \( F \) and an explicit mapping \( \phi: X \rightarrow F \) such that for every pair \( (x_n, z) \), where \( x_n \) is a training data point and \( z \) is an arbitrary data point:
\begin{eqnarray} k(x_n, z) = \langle \phi(x_n), \phi(z) \rangle_F \end{eqnarray}

The above condition ensures that the non-kernelized learning algorithm's performance in the finite-dimensional space \( F \),  during both training and testing, matches its performance in the Hilbert space \( H \).

\begin{theorem}
Suppose \( K \) is a kernel matrix over \( N \) training data points \( \{x_1, \ldots, x_N\} \). Consider the mapping \( \phi: X \rightarrow \mathbb{R}^N \) defined as:

\begin{eqnarray} 
\phi(z) = K^{-1/2} \begin{bmatrix} k(x_1, z) \\ \vdots \\ k(x_N, z) \end{bmatrix} 
\end{eqnarray}.

For every pair \( (x_n, z) \), where \( x_n \) is a training data point and \( z \) is an arbitrary data point, we have:

\begin{eqnarray} \langle \phi(x_n), \phi(z) \rangle = k(x_n, z) \end{eqnarray}
\end{theorem}

\begin{proof}
We have:

\begin{eqnarray}
\begin{aligned} &\langle \phi(x_n), \phi(z) \rangle = \phi(x_n)^T \phi(z) \\&= \begin{bmatrix} k(x_1, x_n) & \cdots & k(x_N, x_n) \end{bmatrix} K^{-1} \begin{bmatrix} k(x_1, z) \\ \vdots \\ k(x_N, z) \end{bmatrix} 
\end{aligned}
\end{eqnarray}

Since \( [k(x_1, x_n), \ldots, k(x_N, x_n)] \) is the \( n \)-th row of matrix \( K \), when this row is multiplied by \( K^{-1}\) it results in a vector that has a value of \(1\) in its \( n \)-th element and \(0\) elsewhere.
Therefore,
\begin{eqnarray}
\begin{aligned} 
&\langle \phi(x_n), \phi(z) \rangle \\
&= [k(x_1, x_n), \ldots, k(x_N, x_n)] K^{-1} [k(x_1, z), \ldots, k(x_N, z)]^T \\
&= k(x_n, z) 
\end{aligned}
\end{eqnarray}

This proves that \( \phi(z) = K^{-1/2} [k(x_1, z), \ldots, k(x_N, z)]^T \) satisfies \( \langle \phi(x_n), \phi(z) \rangle = k(x_n, z) \) for any pair \( (x_n, z) \), demonstrating the desired property of the proposed mapping \( \phi \).
\end{proof}

\section{Kernel PCA analysis using the introduced explicit and exact mapping}
\citep{Scholkopf1998} proposed the kernel version of the PCA algorithm using the dual representation. We have summarized this method in Appendix A for completeness. In this section, by using the proposed explicit and exact mapping, we obtain the kernel PCA algorithm in its primal form. Then, we derive the dual solution in a new way without resorting to the kernel trick. Finally, by combining the primal and dual solutions, we arrive at a new solution that is computationally simpler.

\subsection{Direct PCA analysis in the explicit and exact feature space of a kernel function}

To perform PCA analysis in a feature space, firstly it is necessary to center the mean of the training data in that space. Given that we have the data explicitly in the feature space, we can compute their mean as follows:
\begin{eqnarray} 
\begin{aligned}
M &= \frac{1}{N} \sum_{n=1}^{N} \phi(x_n) \\
&= \frac{1}{N} \sum_{n=1}^{N} K^{-1/2} [k(x_1, x_n), \ldots, k(x_N, x_n)]^T \\ 
&= K^{-1/2} \frac{1}{N} \sum_{n=1}^{N} [k(x_1, x_n), \ldots, k(x_N, x_n)]^T \\
&= K^{-1/2} Ke = K^{1/2} e, 
\end{aligned}
\end{eqnarray}
where \( e \) is an \( N \)-dimensional vector with all elements being \( 1/N \). Let \( \psi \) be a new mapping that centers data besides transferring them to the feature space. We can express \( \psi \) as:
\begin{eqnarray} \psi(z) = \psi(z) - M = K^{-1/2} ([k(x_1, z), \ldots, k(x_N, z)]^T - Ke) \end{eqnarray}

In conventional kernel-based methods, computations are firstly expressed in the feature space using an unknown mapping \( \phi \). Then, the kernel trick is applied, and all instances of inner-product computation in the feature space are replaced with the kernel function. On the other hand, in the new method, all computations are based on the explicit mapping of training and testing data to the feature space by the proposed function \( \psi \). Kernel function computations appear in the explicit mapping, not by applying the kernel trick.
Specifically, for each training data \( x_n \), we have:
\begin{eqnarray}
\begin{aligned}
 \psi(x_n) &= K^{-1/2} ([k(x_1, x_n), \ldots, k(x_N, x_n)]^T - Ke) \\
 &= K^{-1/2} (K_{*,n} - Ke) = K^{1/2} (e_n - e) 
\end{aligned} 
\end{eqnarray}
where \( K_{*,n} \) is the \( n \)-th column of matrix \( K \), and \( e_n \) is an \( N \)-dimensional vector with a 1 in its \( n \)-th element and 0 elsewhere.

To perform PCA analysis, we then compute the covariance matrix in the feature space:
\begin{eqnarray}
\begin{aligned} 
C &= \frac{1}{N} \sum_{n=1}^{N} \psi(x_n) \psi(x_n)^T \\
&= \frac{1}{N} \sum_{n=1}^{N} K^{1/2} (e_n - e) (e_n - e)^T K^{1/2} \\
&= \frac{1}{N} K^{1/2} \left( \sum_{n=1}^{N} (e_n - e) (e_n - e)^T \right) K^{1/2} \\
&= \frac{1}{N} K^{1/2} (I - Nee^T) K^{1/2}
\end{aligned}
\end{eqnarray}

Now, we express the equation \( Cv = \lambda v \) for PCA analysis in the explicit feature space as
\begin{eqnarray}
\label{eq:eigenvalue-problem}
 K^{1/2} (I - Nee^T) K^{1/2} v = N \lambda v .
\end{eqnarray}
Multiplying both sides from the left by \( K^{1/2} \), we obtain:
\begin{eqnarray} K (I - Nee^T) K^{1/2} v = N \lambda K^{1/2} v. \end{eqnarray}
Defining \( u = K^{1/2} v \), we arrive at the following equation:
\begin{eqnarray} K (I - Nee^T) u = N \lambda u.
\end{eqnarray}
It should be noted that in PCA analysis, the vectors \( v \) are normalized, not the vectors \( u \). Therefore, although \( u \) satisfies the above eigenvalue system, its size is not one. The condition \( v^T v = 1 \) implies that:
\begin{eqnarray} v^T v = (K^{-1/2} u)^T (K^{-1/2} u) = u^T K^{-1} u = 1.
\end{eqnarray}
Thus, the size of the vector \( u \) must be chosen such that the above relationship holds.

For a test data \( z \), define $k_z=[k(x_1,z), \ldots, k(x_N,z)]$. Given that we have an explicit mapping to the feature space, the feature corresponding to the principal component \( v \) is given by:
\begin{eqnarray}
\begin{aligned}
\langle \psi(z), v \rangle &= \langle K^{-1/2} (k_z^T - Ke), v \rangle \\
&= \langle K^{-1/2} (k_z^T - Ke), K^{-1/2} u \rangle \\
&= (k_z - e^T K) K^{-1/2} K^{-1/2} u \\
&= k_z K^{-1} u - e^T u 
\end{aligned}
\end{eqnarray}
This formula is similar to the prediction formula in Gaussian processes \citep{Williams2006}.

In this way, we obtained KPCA by directly calculating PCA in the feature space. From a computational perspective, one can even map the entire training and testing dataset to the proposed explicit feature space, and then run a standard implementation of PCA there.

\subsection{PCA analysis in the feature space of a kernel function using dual representation without kernel trick}
If we use the dual representation and consider \( v \) as a linear combination of the training data, we have:
\begin{eqnarray}
\begin{aligned} v &= \sum_{n=1}^{N} \alpha_n \psi(x_n) \\
&= \sum_{n=1}^{N} \alpha_n K^{1/2} (e_n - e) \\
&= K^{1/2} (I - Nee^T) \alpha.
\end{aligned}
\end{eqnarray}

Now, rewriting the eigenvalue-eigenvector problem (\ref{eq:eigenvalue-problem}) using the above relation, we get:
\begin{eqnarray}
\begin{aligned} 
&K^{1/2} (I - Nee^T) K^{1/2} v = N \lambda v \Rightarrow\\
&K^{1/2} (I - Nee^T) K (I - Nee^T) \alpha = N \lambda K^{1/2} (I - Nee^T) \alpha.
\end{aligned}
\end{eqnarray}

Multiplying both sides from the left by \( K^{1/2} \), we arrive at:
\begin{eqnarray} K (I - Nee^T) K (I - Nee^T) \alpha = N \lambda K (I - Nee^T) \alpha. 
\end{eqnarray}

Using the definition \( \bar{K} = K (I - Nee^T) \), the above equation simplifies to:
\begin{eqnarray} \bar{K} \bar{K} \alpha = N \lambda \bar{K} \alpha.
\end{eqnarray}

Similar to \citep{Scholkopf1998}, it can be shown that instead of solving the above equation, it suffices to solve the equation:
\begin{eqnarray} \bar{K} \alpha = N \lambda \alpha. \end{eqnarray}

In this equation, \( \alpha \) must be chosen such that \( v \) is normalized. Therefore, \( \alpha \) must satisfy the following:
\begin{eqnarray}
\begin{aligned} \|v\|^2 &= \langle K^{1/2} (I - Nee^T) \alpha, K^{1/2} (I - Nee^T) \alpha \rangle \\&= \alpha^T (I - Nee^T) K (I - Nee^T) \alpha = 1. 
\end{aligned}
\end{eqnarray}

For a test data point \( z \), the feature corresponding to the principal component \( v \) is computed as follows:
\begin{eqnarray}
\begin{aligned} &\langle \psi(z), v \rangle \\
&= \langle K^{-1/2} (k_z^T - Ke), K^{1/2} (I - Nee^T) \alpha \rangle \\
&= k_z (I - Nee^T) \alpha - e^T K (I - Nee^T) \alpha.
\end{aligned}
\end{eqnarray}

\subsection{Comparing solutions of the primal and dual formulations}
An interesting point is that both the primal and dual forms lead to solving the eigenvalue problem of the matrix \( \bar{K} \). This is intriguing, and it is necessary to demonstrate that both solutions are equivalent. In other words, we need to show that vectors \( u \) and \( \alpha \) are scalar multiples of one another. From the equations \( v = K^{1/2} (I - Nee^T) \alpha \) and \( u = K^{1/2} v \), it follows that:
\begin{eqnarray} u = K (I - Nee^T) \alpha = \bar{K} \alpha = N \lambda \alpha \end{eqnarray}

Therefore, we observe that the vector \( u \) from the primal solution and the vector \( \alpha \) from the dual solution are scalar multiples of one another. Thus, we can express \( \langle \psi(z), v \rangle \) as a combination of the primal and dual solutions, which frees us from computing a mean-centered kernel matrix.

\begin{eqnarray}
\begin{aligned} 
\langle \psi(z), v \rangle &= k_z \alpha - N k_z e e^T \alpha  - N \lambda e^T \alpha \\
&= k_z \left(\alpha - (N e^T \alpha) e\right) - N \lambda e^T \alpha. 
\end{aligned}
\end{eqnarray}

\section{Visualizing feature spaces of kernel functions with t-SNE}
The t-SNE algorithm is highly successful for visualizing high-dimensional spaces. One can express distances between data points in the feature space using a kernel function and modify the t-SNE algorithm to visualize the feature space. The distance between mapped data \( x \) and \( z \) in the feature space can be computed as follows:

\begin{eqnarray} 
\begin{aligned}
&\| \phi(x) - \phi(z) \|^2 \\
&= \langle \phi(x), \phi(x) \rangle + \langle \phi(z), \phi(z) \rangle - 2 \langle \phi(x), \phi(z) \rangle \\
&= k(x, x) + k(z, z) - 2 k(x, z) 
\end{aligned}
\end{eqnarray}

However, our proposed method allows for visualizing the feature space of a kernel function without any changes to the t-SNE code, simply by changing the input data to the algorithm. Having an explicit finite-dimensional feature space enables us to execute t-SNE in the feature space and provide a suitable 2D representation of the arrangement of data in that space.

In this section, we visualize the MNIST data in the feature spaces of various kernel functions. The best result of kernel methods on the MNIST dataset has been achieved by a smart choice of a polynomial kernel, as described on page 341 of \citep{Scholkopf2002}. Here, we compare this polynomial kernel with other usual choices of polynomial kernels. By visualizing the feature spaces of these kernels, we demonstrate that the MNIST data exhibits much better separability in the feature space of the polynomial kernel designed in \citep{Scholkopf2002}. It should be noted that the possibility of visualizing the feature space for kernel functions is provided by the explicit and exact feature mapping introduced in this article.

The MNIST data set consists of 60,000 training samples and 10,000 test samples of digits 0 to 9. The inventors of the t-SNE algorithm used 2,500 samples from the MNIST data for visualizing the data in the input space. Here, we consider digits 2, 4, and 7 from this data set, selecting 500 samples from each of the training and test data sets. Thus, we obtain 1,500 training data and 1,500 test data samples. We use the training data to train t-SNE and visualize the test data using the learned mapping.

The MNIST dataset consists of $28\times 28$ images where the intensity of each pixel ranges from 0 to 255. Initially, the pixel intensities are scaled to a range between 0 and 1 by dividing by 255. Assuming \( x \) and \( z \) are two 784-dimensional vectors representing two images, the following polynomial kernel function can be defined on them:
\begin{eqnarray}
\label{eq:k1}
 k_1(x, z) = \langle x, z \rangle^9, \end{eqnarray}
where we assume \( \langle x, z \rangle \) is defined as follows to ensure its value is always between 0 and 1:
\begin{eqnarray} 
\langle x, z \rangle = \left( \frac{x}{\sqrt{784}} \right)^T \left( \frac{z}{\sqrt{784}} \right) = \frac{1}{784} \sum_{i=1}^{784} x_i z_i 
\end{eqnarray}

Although this approach seems logical and sufficient, excellent results on kernel data have been achieved by support vector machines using a different kernel function. \citep{Scholkopf2002} initially transformed the pixel intensity range to [-1, 1]. Given that most pixels are either black or white, this transformation ensures that most pixel values become approximately 1 or -1. Consequently, the inner product of a vector with itself consistently yields a value close to \( 28 \times 28 = 784 \), making the similarity of each vector with itself almost equal to one when divided by 784. To ensure that the inner product kernel function always has a non-negative value, as is the case for Gaussian kernels, they introduced the following kernel function:
\begin{eqnarray}
\label{eq:k2}
\begin{aligned}
 k_2(x, z) &= \left( \langle 2x - 1, 2z - 1 \rangle + 1 \right)^9 / 512 \\
 &= \left( \left( \langle 2x - 1, 2z - 1 \rangle + 1 \right) / 2 \right)^9, 
 \end{aligned}
\end{eqnarray}
where the pixel intensities are initially transformed to the range [-1, 1] to ensure that \( \langle 2x - 1, 2z - 1 \rangle \) is always between -1 and 1. 

In this section, we visualize the feature spaces of $k_1$ in (\ref{eq:k1}) and $k_2$ in (\ref{eq:k2}) using t-SNE. To achieve this, using the proposed exact and explicit mapping, we explicitly represent the training data in the feature space and then use the t-SNE tool (without any modifications) to visualize data in the feature space. Figure 1 visualizes the test data in the feature space of kernel function \( k_1 \), and Figure 2 visualizes the test data in the feature space of kernel function \( k_2 \). As can be seen, the smart choice of the kernel function, as described on page 341 of \citep{Scholkopf2002}, has significantly contributed to the success of kernel methods on the MNIST dataset. Additionally, Figures 3 and 4 show the results of Fisher kernel discriminant analysis using kernel functions \( k_1 \) and \( k_2 \) on the test data. Once again, it is evident that \( k_2 \) has performed significantly better than \( k_1 \). In our implementation, instead of using kernel Fisher analysis \citep{Baudat2000}, we applied non-kernelized multi-class Fisher analysis \citep{Fukunaga1990} on the explicit feature space of kernel functions.

\section{Conclusions}
In this paper, we introduced an explicit finite-dimensional feature mapping for any desired kernel function. 
The proposed explicit feature space is exact as long as at least one of the two inputs to the kernel function belongs to the training data. Since in both the training and testing phases, one of the inputs to the kernel function is taken from the training data, the proposed explicit mapping is exact for machine learning applications. This finding has significant implications, including:

\begin{enumerate}
\item Solving machine learning problems in the feature space of kernel functions in their primal form without resorting to the dual form.
\item Solving machine learning problems in the feature space in their dual form without resorting to the kernel trick.
\item Explicitly mapping data to the feature space and executing a machine learning algorithm on the mapped data. This enables presenting a kernelized version of algorithms using existing non-kernelized tools, simply by changing their input data.
\end{enumerate}

\begin{figure}
\centering
\includegraphics[width=0.5\textwidth]{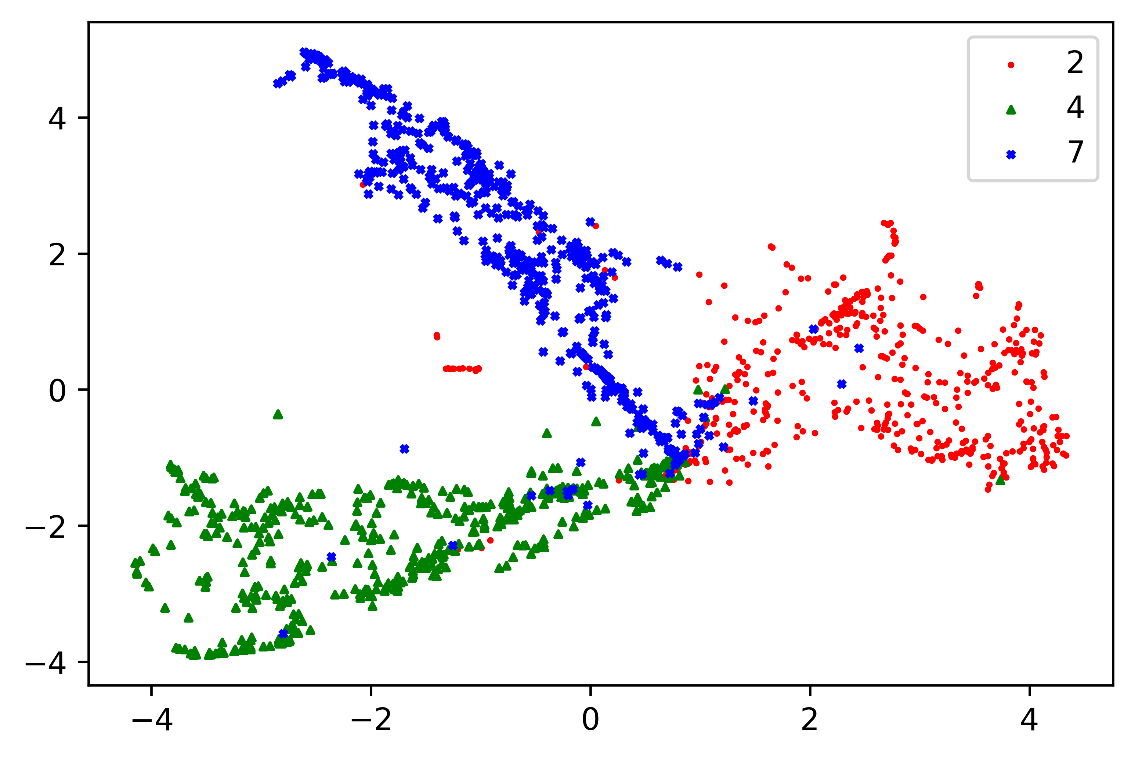}
\caption{Visualization of the feature space using the inappropriate inner product kernel function \( k_1 \) for digits 2, 4, and 7 from the MNIST dataset. \( k_1 \) is an inner product kernel function with a degree of 9 applied to pixels with intensities in the range [0, 1]. Due to the presence of pixels with a value of zero, the similarity of a data point with itself can be very low.}
\label{fig:1}
\end{figure}

\begin{figure}
\centering
\includegraphics[width=0.5\textwidth]{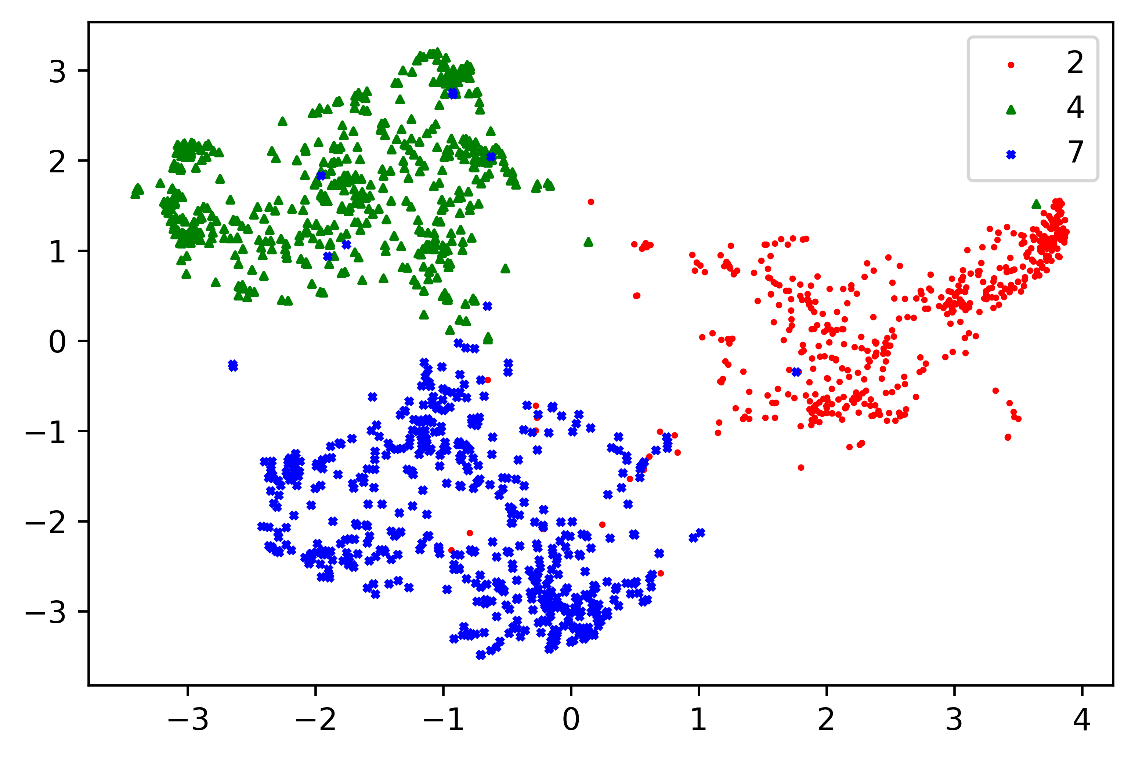}
\caption{Visualization of the feature space using the appropriate inner product kernel function \( k_2 \) for digits 2, 4, and 7 from the MNIST dataset. \( k_2 \) is an inner product kernel function applied to pixels with intensities in the range [-1, 1], and it has been adjusted to function similarly to an RBF kernel, ensuring the similarity of each data with itself is nearly one and the minimum similarity value is zero.}
\label{fig:2}
\end{figure}

\begin{figure}
\centering
\includegraphics[width=0.5\textwidth]{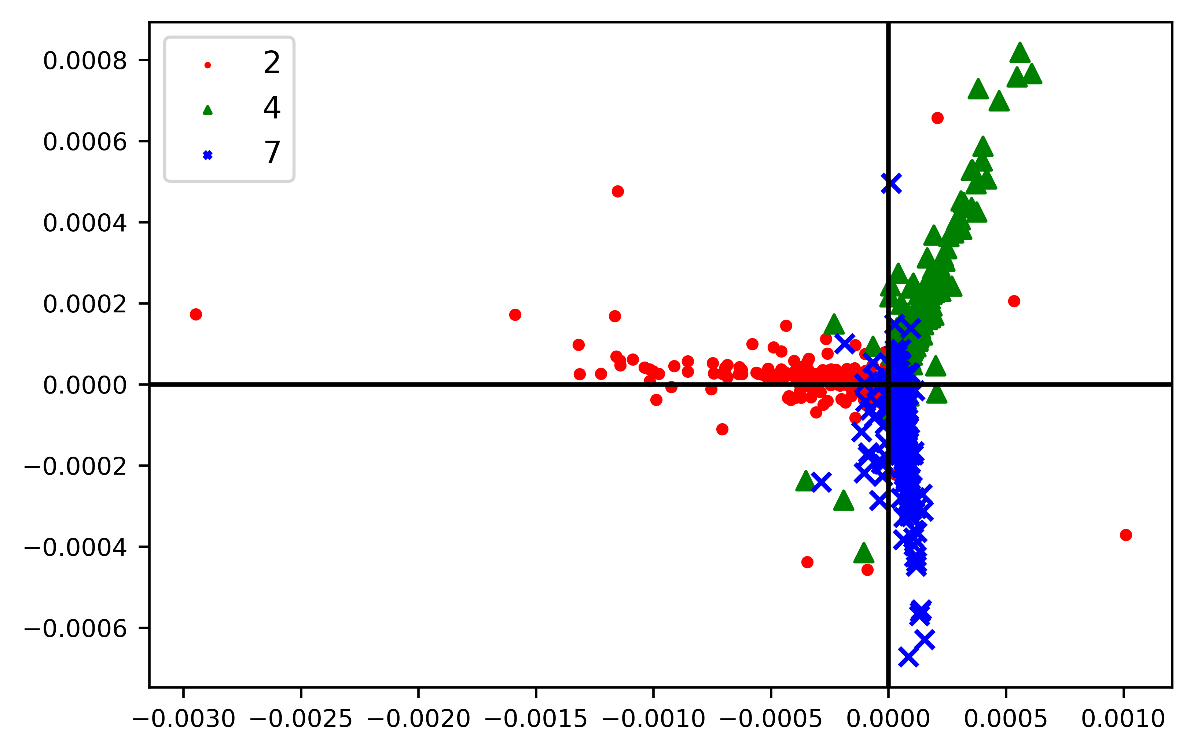}
\caption{Extracted features for the test data using Fisher analysis for the inappropriate inner product kernel function \( k_1 \) applied to digits 2, 4, and 7 from the MNIST dataset. \( k_1 \) is a degree 9 inner product kernel function applied to pixels with intensities in the range [0, 1]. Due to the presence of pixels with a value of zero, the similarity of a data point with itself can be very low.}
\label{fig:3}
\end{figure}

\begin{figure}
\centering
\includegraphics[width=0.5\textwidth]{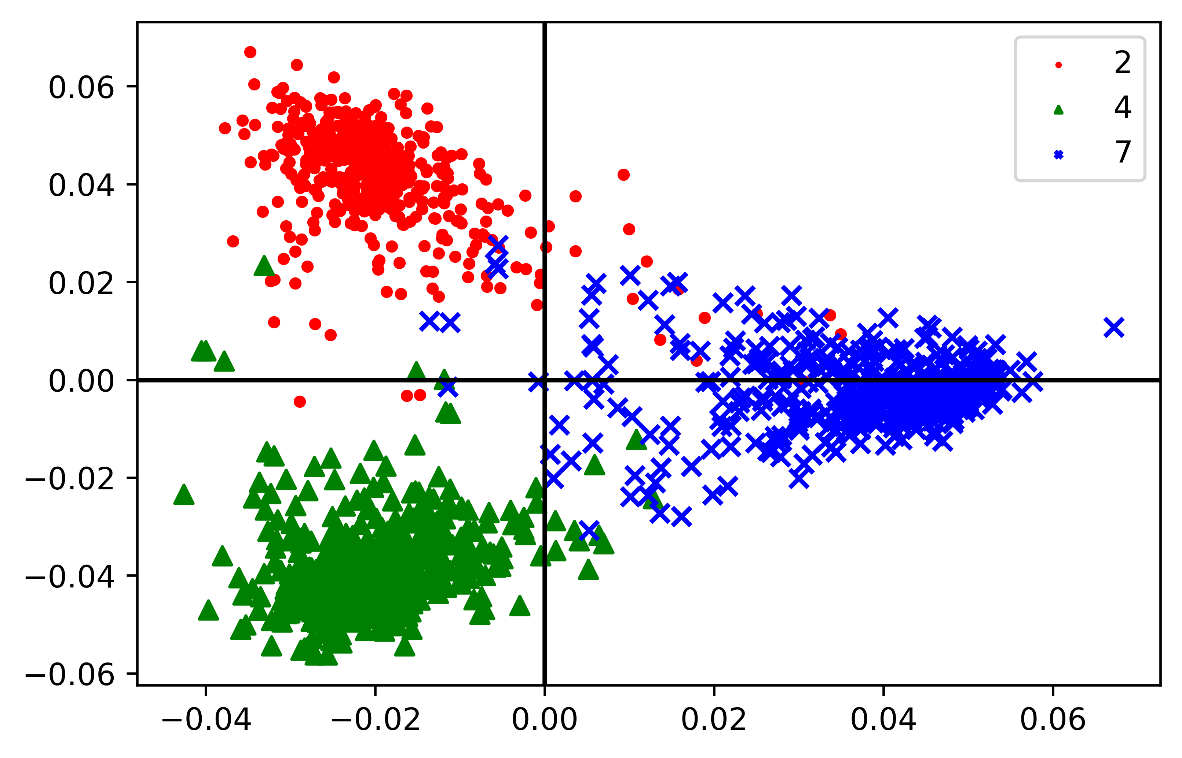}
\caption{Extracted features for the test data using Fisher analysis for the appropriate inner product kernel function \( k_2 \) applied to digits 2, 4, and 7 from the MNIST dataset. \( k_2 \) is an inner product kernel function applied to pixels with intensities in the range [-1, 1]. It has been configured to mimic the behavior of an RBF (Radial Basis Function) kernel, ensuring that each image is highly similar to itself (similarity close to 1) and the minimum similarity value is zero.}
\label{fig:4}
\end{figure}



\newpage


\appendix
\section{Classical method for deriving kernel PCA using kernel trick}
The standard approach \citep{Scholkopf1998} to obtaining the kernel version of the PCA algorithm is as follows. Firstly, it is shown that the eigenvectors can be represented as a linear combination of the training data, yielding a dual representation. Then, assuming that \( \psi \) is a mapping to a feature space where the mean of the training data is zero, an eigenvector $v$ is represented in the feature space as:

\begin{eqnarray}
\label{eq:A1}
 v = \sum_{n=1}^N \alpha_n \psi(x_n).
 \end{eqnarray}

Let the corresponding kernel function for this mapping be \( \bar{k} \). The PCA algorihtm can be formulated in the feature space as:

\begin{eqnarray} \left( \frac{1}{N} \sum_{n=1}^N \psi(x_n) \psi(x_n)^T \right) v = \lambda v.
\end{eqnarray}

Rewriting the above equation using (\ref{eq:A1}), we obtain:

\begin{eqnarray}
\begin{aligned}
&\left( \frac{1}{N} \sum_{n=1}^N \psi(x_n) \psi(x_n)^T \right) \sum_{m=1}^N \alpha_m \psi(x_m) = \lambda \sum_{n=1}^N \alpha_n \psi(x_n)\\
 &\Rightarrow \frac{1}{N} \sum_{n=1}^N \psi(x_n) \sum_{m=1}^N \alpha_m \psi(x_n)^T \psi(x_m) = \lambda \sum_{n=1}^N \alpha_n \psi(x_n)\\
 &\Rightarrow \frac{1}{N} \sum_{n=1}^N \psi(x_n) \sum_{m=1}^N \alpha_m \bar{k}(x_n, x_m) = \lambda \sum_{n=1}^N \alpha_n \psi(x_n)\\
 &\Rightarrow \frac{1}{N} \sum_{n=1}^N \psi(x_n) (K\alpha)_n 
 = \lambda \sum_{n=1}^N \alpha_n \psi(x_n).
\end{aligned}
\end{eqnarray}

Multiplying both sides by \( N \psi(x_l)^T \), we get

\begin{eqnarray}
\psi(x_l)^T \sum_{n=1}^N \psi(x_n) (\bar{K}\alpha)_n = N\lambda \psi(x_l)^T \sum_{n=1}^N \alpha_n \psi(x_n).
\end{eqnarray}

From this, it follows that:

\begin{eqnarray}
\sum_{n=1}^N \bar{k}(x_l, x_n) (\bar{K}\alpha)_n = N\lambda \sum_{n=1}^N \alpha_n \bar{k}(x_n, x_l)
\end{eqnarray}

This result can be expressed in matrix form as:

\begin{eqnarray}
(\bar{K}\bar{K}\alpha)_l = N\lambda (\bar{K}\alpha)_l
\end{eqnarray}

Since the above equation holds for all \( l \) between 1 and $N$, we obtain the following matrix equation:

\begin{eqnarray}
\bar{K}\bar{K}\alpha = N\lambda \bar{K}\alpha.
\end{eqnarray}
It is then sufficient to solve the following equation:
\begin{eqnarray}
\bar{K}\alpha = N\lambda \alpha,
\end{eqnarray}
where:
\begin{eqnarray}
\bar{K} = K - e^T K - Ke + e^T Ke.
\end{eqnarray}

When using the KPCA algorithm to compute the projection of input data \( x \) onto a principal component in the feature space,  \( x \) is first mapped to the zero-mean feature space using \( \psi \), and then multiplied by the principal component:

\begin{eqnarray}
\begin{aligned}
&\langle \psi(x), \sum_{n=1}^N \alpha_n \psi(x_n) \rangle \\
&= \sum_{n=1}^N \alpha_n \langle \psi(x), \psi(x_n) \rangle \\
&= \sum_{n=1}^N \alpha_n \langle \phi(x) - \frac{1}{N} \sum_{m=1}^N \phi(x_m), \phi(x_n) - \frac{1}{N} \sum_{m=1}^N \phi(x_m) \rangle\\
&= \sum_{n=1}^N \alpha_n \left( k(x, x_n) - \frac{1}{N} \sum_{m=1}^N k(x, x_m) \right) \\
&\;\;\;\;- \sum_{n=1}^N \alpha_n \left( \frac{1}{N} \sum_{m=1}^N k(x_m, x_n) + \frac{1}{N^2} \sum_{m=1}^N \sum_{p=1}^N k(x_m, x_p) \right)\\
&= \sum_{n=1}^N \left( \alpha_n - \frac{1}{N} \sum_{m=1}^N \alpha_m \right) k(x, x_n) - e^T K \alpha + e^T K e \alpha^T e.
\end{aligned}
\end{eqnarray}

\end{document}